\documentclass[manuscript,screen]{acmart}
\usepackage{tikz}
\usetikzlibrary{arrows.meta,positioning}
\usepackage{booktabs}
\AtBeginDocument{%
  \providecommand\BibTeX{{\normalfont B\kern-0.5em{\scshape i\kern-0.25em b}\kern-0.8em\TeX}}}
\setcopyright{none}
\renewcommand\footnotetextcopyrightpermission[1]{}
\acmJournal{TIST}

\begin{document}

\title{Evaluating and Guarding Citation Faithfulness in Agentic Scientific Synthesis}

\author{Taewan Goo}
\authornote{Co-first authors.}
\orcid{0000-0001-9427-2290}
\affiliation{\institution{BioNexus}\city{Suwon}\country{Republic of Korea}}
\email{taewan.goo@bionexus.kr}
\author{Junsik Kim}
\authornotemark[1]
\affiliation{\institution{Seoul National University}\city{Seoul}\country{Republic of Korea}}
\author{Kyulhee Han}
\affiliation{\institution{BioNexus}\city{Suwon}\country{Republic of Korea}}
\author{GwonYul Jo}
\affiliation{\institution{BioNexus}\city{Suwon}\country{Republic of Korea}}
\author{Jong-Soo Kim}
\affiliation{\institution{BioNexus}\city{Suwon}\country{Republic of Korea}}
\author{Tae-Hyung Kim}
\authornote{Corresponding author and lead contact.}
\orcid{0000-0001-7280-3978}
\affiliation{\institution{BioNexus}\city{Suwon}\country{Republic of Korea}}
\email{thkim@bionexus.kr}

\renewcommand{\shortauthors}{Goo et al.}

\begin{abstract}
Agentic LLM systems such as OpenScholar and PaperQA2 read the scientific literature and return cited
answers, and both they and their benchmarks already check whether those citations hold, with a fixed
attribution model or human graders. Neither audits the reliability of that check itself. We show it is not
reliable, and that this matters. On identical agent outputs the measured unsupported-citation rate ranges
from about 3\% to about 18\% depending only on the verifier's strictness, and although verifiers agree on
which citations are \emph{supported}, they disagree on which to \emph{flag} (negative-specific agreement 0.27
to 0.30), so no single flag set is trustworthy and cross-paper comparison is invalid without a named verifier
and protocol.

We present a gold-anchored evaluation protocol and a deployable guard that make this behavior measurable and
bounded. The protocol validates the verifier, measures re-attribution, and calibrates a guarantee against
human gold rather than another model's verdict; the verifier is a swappable instrument chosen on cost (recall
0.94 on the supported class, held out), and re-attribution is a commodity step where a deterministic BM25
matches the best open generator. The guard adds a split-conformal layer placing a distribution-free,
finite-sample bound on truly unsupported citations that slip past a chosen flagging rule, a guarantee on
catch rate rather than conclusion correctness. The bound holds on held-out gold, and we identify and quantify
the condition governing its transfer to deployment, calibration-negative difficulty, with a concrete
recalibration recipe, left untested by prior conformal-factuality work.

Validated across four open 27--35B models and three agentic pipelines on public benchmarks (SciFact, QASA,
PubMedQA), with confidence intervals on every headline number, the protocol and guard ship as an open
single-GPU kit.
\end{abstract}

% CCS concepts intentionally omitted at submission; to be generated with the official ACM CCS tool
% (https://dl.acm.org/ccs) and inserted for the camera-ready, per ACM author instructions.


\maketitle

\section{Introduction}
Agentic literature-synthesis systems such as OpenScholar~\cite{asai2026openscholar} and
PaperQA2~\cite{skarlinski2024paperqa2} now operate in a closed loop of retrieval, reasoning, and writing,
returning answers with citations to the scientific record. Such systems are part of a broader move toward
agentic AI for science~\cite{wang2023ai4science,gao2024cell,gottweis2025aicoscientist,ghareeb2025robin}.
For these agents the citation is the trust contract, telling the reader that a specific claim is backed
by a specific source. A growing body of work checks that contract automatically and reports that AI
citations are frequently unsupported~\cite{gao2023alce,wu2025sourcecheckup}, and that even a factually
correct citation may not faithfully reflect the passage the model actually
used~\cite{wallat2024correctness}. The agentic systems themselves do this
too. OpenScholar scores citation support with a fixed attribution model and PaperQA2 with human graders. Yet
these evaluations take the checker for granted, reporting a citation-support number without asking whether
that judgment is itself reliable.

Evaluating this behavior at deployment scale is itself an intelligent-systems problem, because the judgment
must then be made by a model rather than by the human graders a one-off benchmark can afford. We begin from
an uncomfortable observation: \emph{the reported unsupported-citation rate depends on the verifier, and the
verifiers do not agree}. Deciding whether a passage ``supports'' a claim is a
judgment, and automated verifiers make it inconsistently and in disagreement with one another. This is one
instance of a fragility already documented for model graders in
general~\cite{li2024attributionbench,haldar2025rating}: whenever an evaluation hands a judgment to a
model, the reported metric inherits that judge's idiosyncrasies. Three gaps follow for benchmarking agentic
citation faithfulness. \textbf{(G1)} The reported unsupported-citation rate is verifier-dependent, so a
single number does not compare across studies. \textbf{(G2)} Verifiers agree on what is \emph{supported} but
disagree on what to \emph{flag}, so no single deployment flag set can be trusted. \textbf{(G3)} No method
places a distribution-free guarantee on how many truly unsupported citations slip through a chosen flagging
rule. Our contribution is not the observation behind G1 and G2 but its consequence for benchmarking, and a
remedy: anchoring the metric on human gold to make G1 and G2 measurable, and wrapping an imperfect verifier
in a conformal guarantee that closes G3.

If deployment judgments cannot be trusted, the responsible move is to anchor evaluation on human gold
labels, where ``supported'' has a checkable meaning. We do this at three points: we \emph{validate} the
verifier against gold rather than choosing it by reputation; we \emph{measure} re-attribution accuracy
against gold-labeled supporting passages rather than against another verifier; and we \emph{calibrate} a
distribution-free guarantee on gold. The result is both an evaluation protocol and a deployable guard that
wraps an imperfect verifier in a finite-sample bound. Everything ships as an open single-GPU kit.
Figure~\ref{fig:overview} summarizes the design end to end. The result is a new evaluation protocol,
benchmark methodology, and reliability guard for a critical agentic behavior, faithful citation, in
scientific-synthesis agents, treated as an intelligent-systems evaluation problem in which the deployment
metric is itself produced by a model. The work spans formal analysis (a distribution-free finite-sample guarantee
with a proof), evaluation methodology (the gold-anchored protocol), and a deployable system (a single-GPU
lab-audit kit).

\paragraph{Contributions.}
\begin{itemize}
\item \textbf{An agentic-evaluation finding (G1, G2).} On agentic scientific synthesis we quantify how
unreliable deployment citation-faithfulness judgments are: the same agent outputs score about 3\% to about 18\%
unsupported across five gold-validated verifiers, and although the verifiers agree on which citations are
supported, they disagree on which to flag (negative-specific agreement 0.27 to 0.30 among the three
continuous-scored verifiers). That model graders are
imperfect is known; we show it leaves no single deployment flag set trustworthy and makes cross-paper
comparison invalid without a named verifier and protocol.
\item \textbf{A gold-anchored evaluation protocol.} We validate the verifier, measure re-attribution, and
calibrate the guarantee against human gold labels rather than against another verifier's verdict, turning
an unreliable judgment into a checkable measurement.
\item \textbf{A distribution-free guard (G3).} Split-conformal calibration on gold converts an imperfect,
moderate-agreement verifier into a finite-sample bound on unflagged-unsupported citations at a chosen
tolerance; the bound holds empirically on held-out gold, and we show its transfer is governed by
calibration-negative difficulty, so a deployment recalibrates on target-domain negatives rather than
inheriting our threshold.
\item \textbf{An open, reproducible kit.} The protocol and guard run on one 80GB GPU or on CPU and are
released under the MIT license at \url{https://github.com/GooTec/citation-guard}, so a laboratory can audit
an agentic synthesis system on hardware it already has.
\end{itemize}

\begin{figure*}[t]\centering
\includegraphics[width=\textwidth]{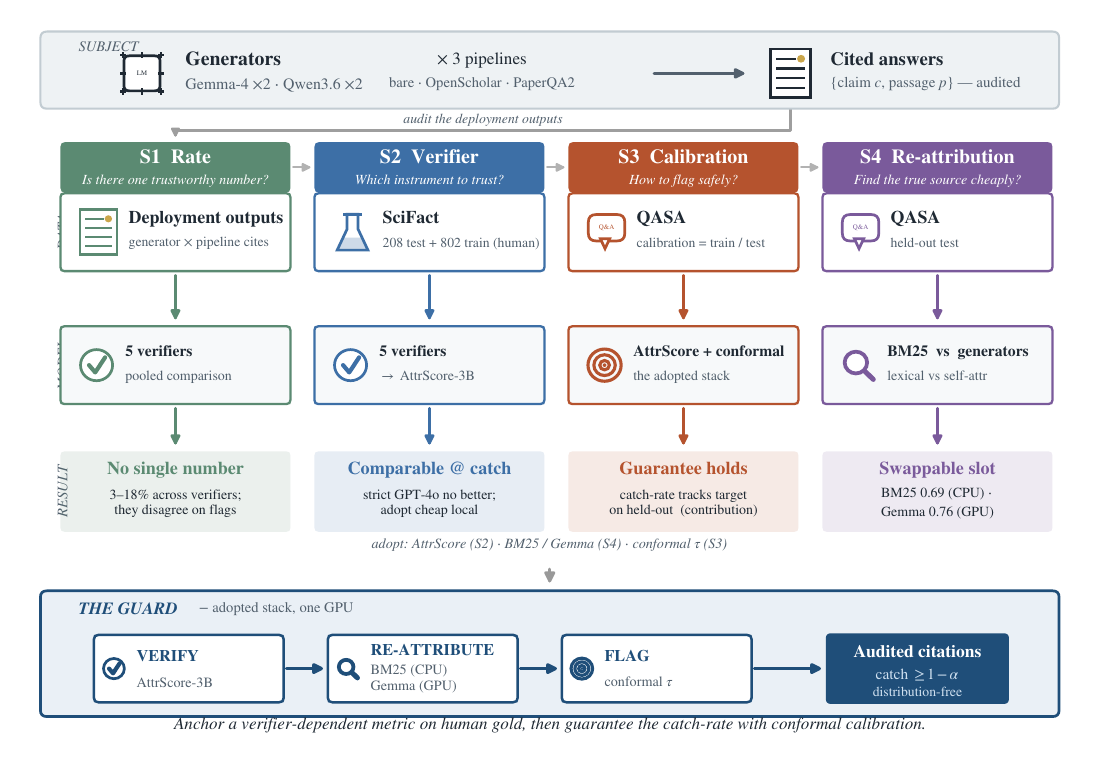}
\caption{Study design. The agentic systems under audit (top) produce cited answers; four evaluation stages
then anchor a verifier-dependent metric on human gold and guarantee the catch rate. \textbf{S1 Rate:} the
unsupported-citation rate has no single value across verifiers. \textbf{S2 Verifier:} on SciFact gold the
verifiers separate comparably at a matched catch rate, so the cheap local AttrScore-3B is adopted on cost.
\textbf{S3 Calibration:} split-conformal calibration on QASA gold turns the verifier score into a
distribution-free catch-rate guarantee (the contribution). \textbf{S4 Re-attribution:} recovering the true
source is a swappable commodity slot (lexical BM25 on CPU, or a generator on GPU). The adopted guard
(verify $\rightarrow$ re-attribute $\rightarrow$ flag) runs on one GPU.}
\label{fig:overview}
\end{figure*}

\section{Related Work}

\subsection{Agentic AI for scientific synthesis}
LLM agents increasingly conduct literature synthesis and discovery, from OpenScholar and
PaperQA2~\cite{asai2026openscholar,skarlinski2024paperqa2} to autonomous chemistry and biology
agents~\cite{boiko2023coscientist,bran2024chemcrow,swanson2025virtuallab,ghareeb2025robin} and
co-scientist systems~\cite{gottweis2025aicoscientist}. These pipelines retrieve, reason, and write in a
loop, often with self-reflection~\cite{shinn2023reflexion,asai2024selfrag}, meta-knowledge-guided
self-correction~\cite{zhang2026selfcorrect}, and corrective retrieval~\cite{yan2024crag}. Their scientific value rests on faithful citation, yet studies warn that
agentic reasoning can also produce confident, well-formed, but unsupported
claims~\cite{messeri2024illusions,wallat2024correctness}.

\subsection{Evaluating and benchmarking agentic systems}
A large benchmarking literature evaluates agents on task success: AgentBench~\cite{liu2023agentbench},
GAIA~\cite{mialon2024gaia}, $\tau$-bench~\cite{yao2024taubench}, AgentBoard~\cite{ma2024agentboard}, tool
use~\cite{patil2024bfcl,guo2024stabletoolbench}, and software engineering~\cite{jimenez2024swebench}, and
surveys chart LLM and LLM-agent evaluation broadly~\cite{chang2024evalsurvey,guan2026agenteval}. A recent
certification scheme issues assurances for LLM-based applications~\cite{bena2026certification}. These measure
whether an agent reaches a goal or certify high-level application properties. Our concern is orthogonal and
finer-grained: not whether the agent succeeds, but whether the evidence it cites actually supports what it
writes, and crucially whether \emph{that} measurement is itself reliable. Where a certification scheme attests
arbitrary application-level properties, we place a distribution-free, finite-sample bound on one concrete,
passage-level quantity, the unsupported-citation slip-through rate, and treat the reliability of the metric
as a first-class object.

\subsection{LLM-as-judge reliability}
Model-graded evaluation is now pervasive~\cite{zheng2023mtbench,kim2024prometheus2}, and so are its known
failure modes: self-preference~\cite{panickssery2024selfpref}, sensitivity to spurious features such as
length~\cite{dubois2024lengthcontrol}, self-inconsistency~\cite{haldar2025rating}, and the general
difficulty of automatic attribution evaluation~\cite{li2024attributionbench}. Proposed mitigations include
specialized judges~\cite{kim2024prometheus2}, reward-model benchmarks~\cite{lambert2024rewardbench},
judge benchmarks~\cite{tan2024judgebench}, and juries of diverse models~\cite{verga2024juries}. We add a
quantitative account of judge disagreement specifically for citation support on agentic outputs, and,
rather than seek a better judge, we anchor the metric on gold and bound the residual error. A jury of the
disagreeing verifiers does not resolve this: averaging idiosyncratic thresholds still has no ground-truth
anchor, and a majority-vote flag set inherits the same untrustworthy operating point we document; gold
anchoring supplies the reference that consensus cannot.

\subsection{Citation and attribution evaluation}
Attribution evaluation spans the AIS framework~\cite{rashkin2023ais}, ALCE~\cite{gao2023alce},
TRUE~\cite{honovich2022true}, AttrScore~\cite{yue2023attrscore}, RAGAS~\cite{es2024ragas},
CiteEval~\cite{xu2025citeeval}, and medical citation auditing~\cite{wu2025sourcecheckup}. A separate line
verifies whether a cited reference \emph{exists} and whether its metadata is correct, catching fabricated
references~\cite{shi2026citeaudit,lee2026citecheck,chelli2024hallucination}; our concern is the orthogonal
one, whether a genuinely existing passage \emph{supports} the claim attached to it. AttributionBench shows
this support judgment is hard even for strong models~\cite{li2024attributionbench}, which motivates
treating any single verifier's verdict as uncertain.

\subsection{Conformal and selective prediction}
Conformal prediction provides distribution-free, finite-sample guarantees~\cite{angelopoulos2021conformal}
and has been applied to language-model factuality~\cite{mohri2024conformal}, conformal language
modeling~\cite{quach2024conformal}, and hallucination abstention~\cite{abbasiyadkori2024conformal}, in the
tradition of selective prediction~\cite{elyaniv2010selective,geifman2017selectivenet}. The closest work
bounds answer-level claim correctness; we instead place a distribution-free bound on the passage-level
unsupported-\emph{citation} slip-through rate, a quantity specific to citation faithfulness.

\section{Preliminaries and Problem Formulation}

\subsection{Notation and definitions}
An agentic synthesis system answers a query by retrieving a set of passages $P = \{p_1, \dots, p_m\}$ and
returning an answer composed of cited sentences, which we decompose into claim--citation pairs $(c, p_i)$
where claim $c$ is a sentence and $p_i \in P$ is the passage it cites. Table~\ref{tab:notation} collects the
notation.

\begin{definition}[Verifier and support]\label{def:verifier}
A \emph{verifier} is a function $v(c,p) \in \{0,1\}$ that judges whether passage $p$ supports claim $c$,
optionally exposing a continuous score $s(c,p) \in [0,1]$ that estimates the probability of support. A
citation $(c,p_i)$ is \emph{unsupported under $v$} if $v(c,p_i)=0$. A \emph{gold} label $y(c,p)\in\{0,1\}$ is
a human annotation of support, available only on benchmark data.
\end{definition}

\begin{definition}[Unsupported-citation rate]\label{def:rate}
The \emph{unsupported-citation rate} of a system under verifier $v$ is the fraction of its claim--citation
pairs that are unsupported under $v$. As a function of both the outputs and $v$, it is meaningful only
together with the verifier and protocol that produced it.
\end{definition}

\begin{table}[t]
\caption{Notation used throughout.}
\label{tab:notation}
\begin{tabular}{ll}
\toprule
Symbol & Meaning \\
\midrule
$c$ & a claim (a generated sentence) \\
$P=\{p_1,\dots,p_m\}$ & passages retrieved for the query \\
$(c,p_i)$ & claim--citation pair; $p_i\in P$ is the cited passage \\
$v(c,p)\in\{0,1\}$ & verifier verdict ($1$ = supports) \\
$s(c,p)\in[0,1]$ & continuous verifier score (probability of support) \\
$y(c,p)\in\{0,1\}$ & human gold support label (benchmark only) \\
$\alpha$ & tolerance; the target catch rate is $1-\alpha$ \\
$\tau$ & flag threshold; flag the citation if $s(c,p)\le\tau$ \\
$n_{\text{cal}}$ & number of calibration (unsupported-class) scores \\
\bottomrule
\end{tabular}
\end{table}

\subsection{Problem statement}
Let a deployed agent produce claim--citation pairs whose true support status is unknown. A \emph{gold}
label $y(c,p) \in \{0,1\}$ is a human annotation of support, available only on benchmark data. Two problems
follow. (1) \emph{Measurement.} The reported unsupported rate is a function of both the outputs and the
verifier; we ask how much it varies with the choice of $v$ and how much verifiers agree on the per-citation
verdict. (2) \emph{Guarantee.} Given an imperfect verifier, we seek a flagging rule that, for a chosen
tolerance $\alpha$, catches at least a $1-\alpha$ fraction of the truly unsupported citations, with a
finite-sample, distribution-free guarantee calibrated on gold. The first problem establishes that no single
verifier verdict can be trusted; the second turns that imperfect verifier into a bounded triage.

\section{Method}

\subsection{The verifier as a swappable measurement instrument}
Any verifier is an automated judge that emits a continuous score $s(c,p)$ separating supported from
unsupported claims, and different verifiers make different errors, so the verifier must be validated against
gold before we trust anything it reports. Because the conformal layer below is what controls the catch rate
on unsupported citations, the verifier need not be the strictest judge, only a good separator. We therefore
compare gold-validated verifiers at a \emph{matched catch rate} and, finding their separation comparable,
adopt the attribution-tuned 3B model AttrScore~\cite{yue2023attrscore} on cost because it runs locally. Its
separation on gold is the signal the conformal layer turns into a guarantee (Figure~\ref{fig:verifier}).

\subsection{Quantifying verifier disagreement}
To measure how much the metric depends on the instrument, we score identical agent outputs with a panel of
gold-validated verifiers spanning the frontier, from high-recall AttrScore to a high-specificity
natural-language-inference (NLI) judge, with RAGAS~\cite{es2024ragas} and a frontier judge in between. We
report (i) the unsupported rate under each verifier and (ii) class-specific pairwise agreement on a fixed
sample of deployment citations. Because the supported class is highly prevalent, we read
\emph{negative-specific} agreement (agreement on what to flag) rather than Cohen's $\kappa$ alone, which
the base rate depresses (the kappa paradox; Supplementary~S2).

\subsection{The conformal guard}
The guard processes each cited sentence in three steps (Figure~\ref{fig:overview}). It \emph{verifies} the
claim against its cited passage with the 3B verifier; if unsupported, it attempts to \emph{re-attribute}
the claim to another provided passage that supports it, keeping the claim and fixing only the pointer; what
cannot be repaired it \emph{flags} rather than deletes. The verifier is imperfect, so a hard threshold
would have unknown error on new data. Split conformal removes that uncertainty. On a held-out gold
calibration set we choose a tolerance $\alpha$ and set the flag threshold $\tau$ to the appropriate
quantile of the unsupported-class scores with a finite-sample correction:
\begin{equation}
\tau = s_{(k)}, \qquad k = \big\lceil (n_{\text{cal}}+1)(1-\alpha) \big\rceil,
\end{equation}
where $s_{(1)} \le \dots \le s_{(n_{\text{cal}})}$ are the sorted scores of the calibration unsupported
citations. Flagging every citation with $s(c,p) \le \tau$ then catches at least $1-\alpha$ of the truly
unsupported citations, in finite samples and without any assumption on the score distribution; the only
assumption is exchangeability between calibration and deployment.

\begin{proposition}[Finite-sample catch rate]\label{prop:coverage}
Let the calibration unsupported-class scores $s_1,\dots,s_{n_{\text{cal}}}$ and a fresh unsupported-class
score $s_{\mathrm{new}}$ be exchangeable, and set $\tau=s_{(k)}$ with $k=\lceil(n_{\text{cal}}+1)(1-\alpha)\rceil$.
Then the rule ``flag if $s\le\tau$'' catches the new unsupported citation with probability at least $1-\alpha$:
$\Pr[\,s_{\mathrm{new}}\le\tau\,]\ge 1-\alpha$.
\end{proposition}
\begin{proof}
By exchangeability the rank of $s_{\mathrm{new}}$ among the $n_{\text{cal}}+1$ scores is uniform on
$\{1,\dots,n_{\text{cal}}+1\}$, so $\Pr[\,s_{\mathrm{new}}\le s_{(k)}\,]\ge k/(n_{\text{cal}}+1)\ge 1-\alpha$
by the choice of $k$. This is the standard split-conformal argument~\cite{angelopoulos2021conformal} applied
to the unsupported-class score distribution; it gives coverage $\ge 1-\alpha$ under exchangeability and is
exact for continuous, tie-free scores. The discrete verifier score can tie; we break ties conservatively
(treating $s_{\mathrm{new}} \le \tau$ at equality), which preserves the $\ge 1-\alpha$ bound. This is what
earns the word ``guarantee'' even when gold is scarce.
\end{proof}
The bound is on the catch rate, the safety axis; we verify below that it holds empirically on held-out gold
(Figure~\ref{fig:guard}A).

\subsection{Re-attribution}
When a citation is flagged, the guard ranks the other provided passages, proposes the top-ranked one as a
replacement pointer, and re-verifies that candidate before the pointer is moved. The ranker is a swappable
slot: we adopt a deterministic lexical BM25 by default and compare it against the verifier's own attribution
score $s(c,p_j)$ and against the generators asked to re-attribute their own citations. We evaluate
re-attribution against gold: on data where a human-labeled supporting passage exists, we measure how often
the gold passage is ranked first or within the top $k$, scored against truth rather than against the
verifier that did the flagging. Because the ranker only proposes and the verifier disposes, a lexical
default is safe even though it cannot itself distinguish support from contradiction.

\section{Experiments}

\subsection{Setup}
\paragraph{Agents and conditions.} We audit four open instruction-tuned models in the 27--35B range, two
dense and two mixture-of-experts from the Gemma-4 and Qwen3.6 families~\cite{gemma4_2026,qwen36_2026}, run
locally on a single GPU, across two published agentic pipelines, OpenScholar and PaperQA2, against a bare
no-retrieval floor. We hold the retrieved context fixed and reproduce each system's published prompting, so
the rates reflect generation-time citation behavior rather than each system's own retrieval. Running every
generator locally keeps the audit fully reproducible on commodity hardware. The deployment citations we
audit for the rate and disagreement analyses are these systems' answers on the ScholarQABench multi-paper
tasks~\cite{asai2026openscholar}.

\paragraph{Verifiers.} Five gold-validated verifiers span the recall--specificity frontier: AttrScore-3B
(high recall), an OpenScholar post-hoc judge, GPT-4o, RAGAS, and a DeBERTa-NLI judge (high specificity).

\paragraph{Gold data.} Verifier validation uses SciFact ($n=208$ human-labeled claim--passage test pairs,
with a disjoint $n=802$ split used only for prompt selection). Re-attribution and conformal calibration use
the full QASA test set ($n=1375$), a question-answering benchmark over research papers~\cite{lee2023qasa}
with human-labeled supporting passages. A scope control uses PubMedQA ($n=843$).

\paragraph{Metrics.} We report per-class recall and specificity (not a single accuracy, which moves with
class balance), Cohen's $\kappa$ with bootstrap 95\% confidence intervals (CIs), unsupported-citation
rates, re-attribution recall@$k$ with binomial 95\% CIs, and conformal catch rate against the declared
target.

\subsection{The verifier, a swappable instrument chosen on cost}
\begin{figure}[tbp]\centering
\includegraphics[width=\linewidth]{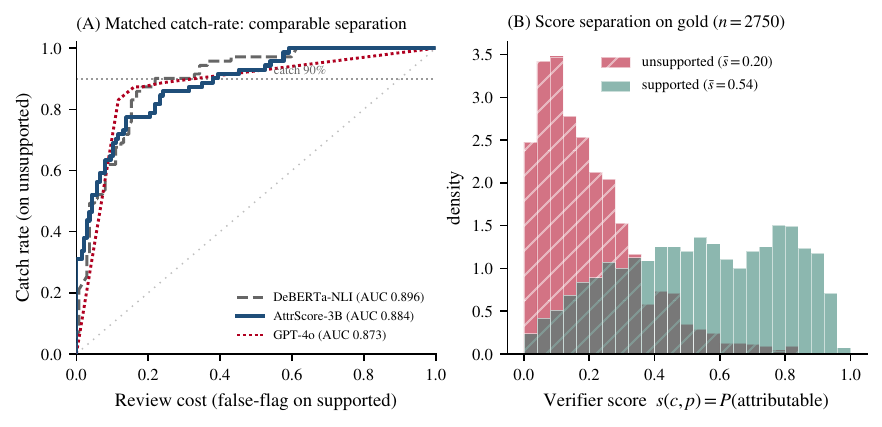}
\caption{The verifier as a swappable instrument, validated on gold. (A) Matched catch-rate frontier on
SciFact: at any common catch rate on unsupported citations the three verifiers separate comparably (AUC and
review cost overlap), so a strict frontier judge buys no advantage and the cheap local AttrScore-3B is
adopted on cost. (B) The continuous score separates supported from unsupported claims on gold (means 0.54
versus 0.20), the signal the conformal layer turns into a guarantee.}
\label{fig:verifier}
\end{figure}
Because the conformal layer controls the catch rate, the verifier need only separate the two classes well,
so we compare verifiers at a \emph{matched catch rate} rather than at their native thresholds. On SciFact
gold the continuous-scored verifiers separate comparably (Figure~\ref{fig:verifier}A): their ROC-AUC values
overlap (AttrScore-3B 0.88 [0.84, 0.93], DeBERTa-NLI 0.90 [0.85, 0.94], GPT-4o 0.87 [0.82, 0.92]; bootstrap
95\% CIs, SciFact $n=208$), and at any common catch rate their review cost is comparable, so a strict
frontier judge such as GPT-4o buys no separation advantage over the local 3B model. Given comparable separation, the verifier is a
swappable instrument and the deciding factor is cost: we adopt AttrScore-3B because it runs locally on one
GPU. Its native operating point is high recall on the supported class (0.90 on SciFact,
Table~\ref{tab:verifier}; 0.94 on a disjoint held-out split, $n=802$), which suits a guard that must not
discard genuine citations, while the per-class trade-offs of the stricter verifiers (DeBERTa recall
0.25/specificity 0.97; GPT-4o recall 0.46) explain the inflated rates they report in
Figure~\ref{fig:decay}. Supplementary~S1 gives the full per-verifier metrics, the matched-catch comparison,
the prompt-selection procedure, and the held-out validation.

\begin{table}[t]
\caption{The verifier frontier on SciFact gold ($n=208$): recall on the supported class (the safety axis)
versus specificity. AttrScore-3B sits at the high-recall end. The adopted verifier is re-validated on a
disjoint held-out split (recall 0.94, $n=802$); full metrics, Cohen's $\kappa$ with CIs, and the
prompt-selection comparison are in Supplementary~S1.}
\label{tab:verifier}
\begin{tabular}{lcc}
\toprule
Verifier & Recall (supported) & Specificity \\
\midrule
AttrScore-3B (adopted) & \textbf{0.90} & 0.54 \\
OpenScholar post-hoc   & 0.50 & 0.92 \\
GPT-4o                 & 0.46 & 0.92 \\
RAGAS                  & 0.30 & 0.96 \\
DeBERTa-NLI            & 0.25 & 0.97 \\
\bottomrule
\end{tabular}
\end{table}

\subsection{The unsupported rate is verifier-dependent}
\begin{figure}[tbp]\centering
\includegraphics[width=0.7\linewidth]{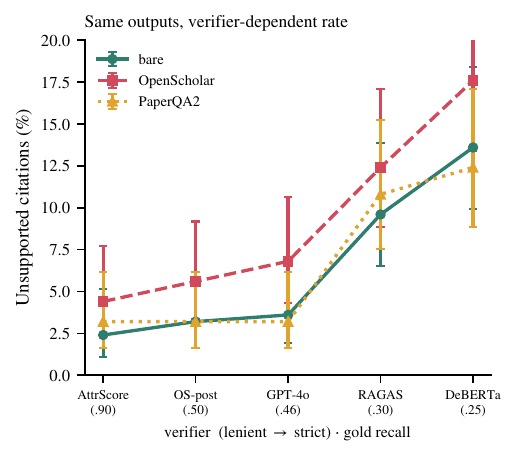}
\caption{The unsupported-citation rate has no single value: on the same agent outputs it ranges threefold
to sixfold across five gold-validated verifiers ordered by recall. Most of the climb is the strict
verifiers falsely rejecting genuinely supported citations.}
\label{fig:decay}
\end{figure}
On the same agent outputs the measured unsupported rate rises monotonically as the verifier gets stricter,
from about 3\% at the bare floor under the high-recall verifier to roughly 14--18\% under a strict NLI
judge, with the agentic pipelines a few points above the floor at each setting (Figure~\ref{fig:decay}).
Much of that climb is not extra unfaithfulness but extra false rejection: the strict verifiers have low
recall on genuinely supported claims (0.25 on gold), so their high rates partly count supported citations.
The rates are pooled over the four open models and three pipelines and stay single-digit at the lenient end
while climbing with verifier strictness, so the effect is a property of the verifier, not of any one
generator. We therefore report no single rate without naming its operating point; two unsupported rates from
different papers are not comparable unless they name the same verifier and protocol.

\subsection{Verifiers disagree on what is unsupported}
Disagreement runs deeper than the headline rate, and it is specifically about what to \emph{flag}. On 300
cited sentences from the deployed pipelines, three gold-validated verifiers spanning the
recall--specificity range (AttrScore-3B, DeBERTa-NLI, and the GPT-4o judge; the two binary-only judges,
RAGAS and OpenScholar post-hoc, are omitted from this pairwise analysis) agree strongly on which
citations are supported (positive-specific agreement 0.90--0.96) but agree on which to flag as unsupported
only weakly (negative-specific agreement 0.27--0.30); concretely, 60 to 80\% of the citations one verifier
flags, another calls supported. (Full agreement statistics, including observed agreement, Gwet's AC1, and Cohen's $\kappa$ with
the base-rate caveat and CIs, are in Supplementary~S2.) The verifiers thus encode genuinely different flag
thresholds: there is no single deployment flag set to trust. Everything load-bearing below, the verifier's
validation, the re-attribution accuracy, and the conformal threshold, is therefore anchored on human gold
rather than on a model's idiosyncratic cutoff.

\subsection{The conformal guarantee holds in-distribution}
\begin{figure}[tbp]\centering
\includegraphics[width=\linewidth]{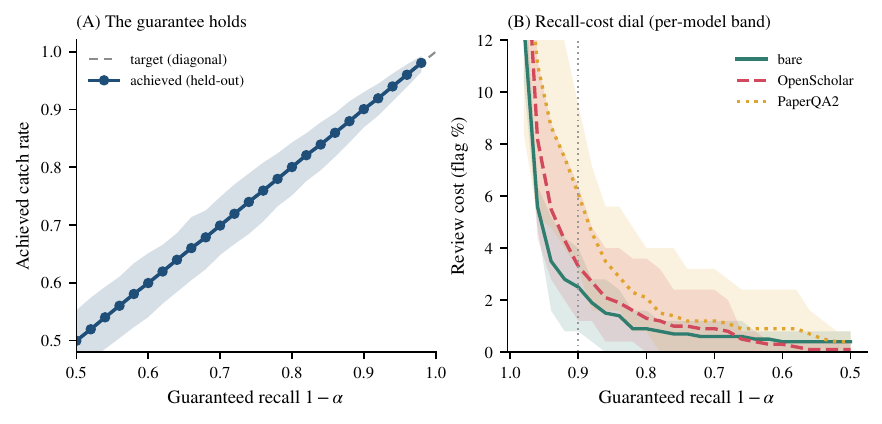}
\caption{The conformal guarantee, calibrated on QASA gold ($n=1375$). (A) The guarantee holds: the achieved
catch rate tracks the declared target $1-\alpha$ on repeated held-out splits across the operating range
(target 0.90/0.94/0.98 $\rightarrow$ achieved 0.90/0.94/0.98). (B) The recall-cost dial, per open model:
declaring a guaranteed catch rate fixes the review budget in advance. The verify $\rightarrow$ re-attribute
$\rightarrow$ flag pipeline itself is shown in Figure~\ref{fig:overview}.}
\label{fig:guard}
\end{figure}
Calibrated on QASA gold (the full test set, 1375 supported and 1375 unsupported pairs), the conformal flag rule's achieved
catch rate tracks the declared target across the operating range on repeated held-out splits
(Figure~\ref{fig:guard}A), so the distribution-free bound is empirical, not only asymptotic. This coverage
is \emph{marginal} over the calibration draw: an individual split can fall below target, which is why we
calibrate on as much gold as is available and recalibrate when the target setting differs. The
finite-sample correction is what earns the word ``guarantee'' when gold is scarce: at a calibration size
of 25 a naive empirical-quantile threshold under-covers (0.85 at a 0.90 target) while the conformal rule
holds (0.92); the gap closes by a calibration size of about 200 (Supplementary~S3). At large calibration
sizes the two coincide, so the value of conformal here is distribution-free validity, not a numerical
gain.

\paragraph{Calibration-negative difficulty governs transfer.} The bound is distribution-free only under
exchangeability between calibration and deployment negatives, and that assumption is consequential. Our
calibration negatives are QASA in-paper distractors, which are easy: the verifier scores them low (mean
$P(\text{attributable})=0.20$). Genuinely mis-cited claims are harder, because the cited passage is on
topic but does not support the specific claim; on SciFact's human-labeled unsupported claims the verifier
scores them far higher (mean $0.50$). A threshold $\tau$ calibrated on the easy distractors to catch
$90\%$ of them therefore catches only $37\%$ of the harder human-labeled negatives ($44\%$ and $61\%$ at
the $94\%$ and $98\%$ targets; Supplementary~S6). The conformal machinery is not at fault: recalibrating
$\tau$ on negatives drawn from the deployment distribution restores the guarantee, at a higher review
budget. The operational rule is therefore that the calibration set must resemble the deployment negatives;
the QASA coverage above is an in-distribution demonstration, and a deployment should recalibrate on
target-domain negatives rather than inherit the QASA threshold. Identifying and quantifying this
calibration-negative-difficulty condition is itself part of the contribution: prior conformal-factuality
work names distribution shift as a limitation but does not characterize which shift matters, whereas we
isolate one such condition empirically and show how to restore the guarantee.

The guarantee pays for safety in review effort, the genuine citations it also flags
(Figure~\ref{fig:guard}B). At a 90\% guarantee, pooled across the four open models, the flag rate at the high-recall
operating point is about 3\% for OpenScholar and 6\% for PaperQA2, against about 2.5\% for the bare model:
a few cited sentences per hundred sent for review. Table~\ref{tab:permodel} breaks the budget down by model:
it stays single-digit at the bare floor for every open model and rises with pipeline agency, with no single
model driving the pooled figure.

\begin{table}[t]
\caption{Per-model review budget: flag rate (\%) at a 90\% guaranteed catch rate, by generator and pipeline
(QASA-calibrated $\tau$). The budget stays single-digit at the bare floor for every open model and rises with
pipeline agency; no single model drives the pooled rate. Each cell is $n=250$ deployment citations; Wilson
95\% CIs span a few points (e.g. bare Gemma-4-26B $0.8\%$ [0.2, 2.9], PaperQA2 Qwen3.6-35B $9.5\%$ [6.5, 13.8]).}
\label{tab:permodel}
\begin{tabular}{lccc}
\toprule
Generator & bare & OpenScholar & PaperQA2 \\
\midrule
Gemma-4-31B & 3.2 & 2.8 & 4.0 \\
Gemma-4-26B & 0.8 & 1.2 & 3.2 \\
Qwen3.6-27B & 2.0 & 6.4 & 7.6 \\
Qwen3.6-35B & 4.0 & 2.8 & 9.5 \\
\midrule
Pooled      & 2.5 & 3.3 & 6.1 \\
\bottomrule
\end{tabular}
\end{table} Even a single-digit unsupported rate matters at
synthesis scale, where one survey can carry hundreds of cited sentences; the value is not a large defect
rate but a bounded, declared review budget over many citations.

\subsection{Re-attribution is a swappable, commodity slot}
\begin{figure}[tbp]\centering
\includegraphics[width=\linewidth]{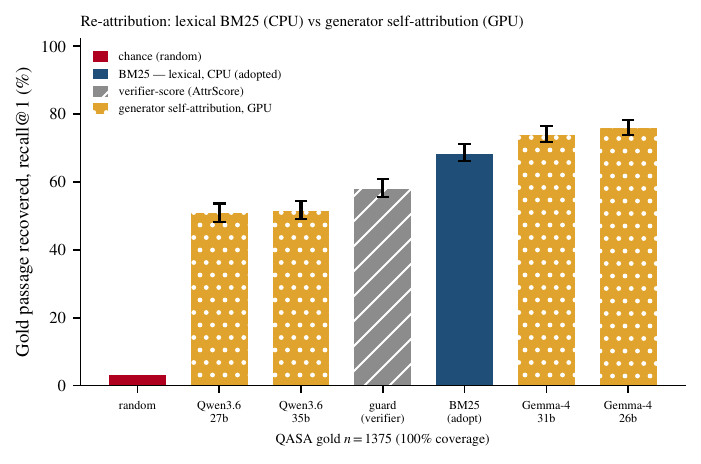}
\caption{Re-attribution validated against gold (QASA, $n=1375$, recall@1, 100\% coverage). Recovering the
supporting passage is a swappable commodity slot: a deterministic lexical BM25 on CPU (0.69) matches the
verifier score (0.58) and is reliable, while the open generators run locally on GPU vary from 0.51 (Qwen) to
0.76 (Gemma). All crush the 3\% random baseline. Wilson 95\% CIs.}
\label{fig:reattr}
\end{figure}
When a citation is flagged, recovering its true supporting passage among the provided candidates is an easy
retrieval task, and it is not where the contribution lies. On QASA gold ($n=1375$; Table~\ref{tab:reattr},
Figure~\ref{fig:reattr}) every reasonable ranker crushes the 3\% random baseline and they form a clear
order: a deterministic lexical BM25 reaches recall@1 0.69 on CPU, the verifier's own attribution score
reaches 0.58, and the open generators asked to re-attribute their own citations range from 0.51 (Qwen3.6) to
0.74--0.76 (Gemma-4) on GPU. The most accurate ranker is therefore a generator, not BM25 and not the
verifier; but generator quality is model-dependent, and the advantage over BM25 is concentrated where it is
least needed. A lexical-overlap stratification (Supplementary~S4) makes this concrete: on the half of items
whose answer shares vocabulary with the gold passage, BM25 and Gemma are tied (0.88 each); the generator's
edge appears only on the low-overlap tail (0.31 versus 0.53 at the lowest-overlap quartile), exactly the
cases that need reading for entailment rather than matching words.

Two properties make BM25 the sensible default for this slot despite not being the most accurate. It is
deterministic and free, with no model to host, 100\% coverage, and reproducible output; and re-attribution
is a retrieval step, not a support judgment, so a lexical ranker is appropriate. Crucially, a lexical ranker
cannot tell support from contradiction: on SciFact an answer and a refuting passage share almost identical
vocabulary (Jaccard 0.064 versus 0.055), so BM25 would surface a contradicting passage as readily as a
supporting one, whereas the verifier separates them (mean $P(\text{attributable})$ 0.83 versus 0.50;
Supplementary~S4). The division of labor is therefore deliberate: the re-attributor proposes a candidate by
lexical match, and the verifier then re-checks that the proposed passage actually supports the claim before
the pointer is moved. Because the slot is swappable, an operator already running a strong generator can plug
it in for the low-overlap tail; either way re-attribution is a commodity and the guard's load-bearing value
is flagging with a guarantee, not repair. On real deployed outputs the picture is soberer still: most
genuinely flagged claims are unsupported by \emph{any} provided passage, so there is no better source to move
them to.

\begin{table}[t]
\caption{Re-attribution recall@$k$ on QASA gold ($n=1375$, 100\% coverage). The open generators (run locally
on GPU) are most accurate but vary by model; a deterministic lexical BM25 on CPU is close and reliable; all
crush chance. Wilson CIs, the lexical-overlap stratification, and the paired McNemar test (BM25 vs.\ verifier
score) are in Supplementary~S4.}
\label{tab:reattr}
\begin{tabular}{llccc}
\toprule
Re-attribution ranker & resource & recall@1 & recall@3 & recall@5 \\
\midrule
Gemma-4-26B (self-attribution)  & GPU & \textbf{0.76} & \textbf{0.86} & \textbf{0.87} \\
Gemma-4-31B (self-attribution)  & GPU & 0.74 & 0.84 & 0.86 \\
BM25 (lexical, adopted default) & CPU & 0.69 & 0.83 & 0.87 \\
Guard (verifier score)          & GPU & 0.58 & 0.75 & 0.82 \\
Qwen3.6-35B (self-attribution)  & GPU & 0.52 & 0.59 & 0.59 \\
Qwen3.6-27B (self-attribution)  & GPU & 0.51 & 0.58 & 0.59 \\
Random                          &     & \multicolumn{3}{c}{0.03} \\
\bottomrule
\end{tabular}
\end{table}

\subsection{Worked examples}
Two cases from the deployed outputs (Gemma-4 under PaperQA2 and OpenScholar) illustrate the guard's two
non-trivial actions. \emph{Re-attribution.} For the claim ``these attacks are used to identify sensitive
information leakage,'' the cited passage only states that word embeddings are trained on potentially
sensitive data, so the verifier finds it unsupported; BM25 then proposes another provided passage (``we
quantitatively investigate how machine learning models leak information$\ldots$ the membership inference
attack''), the verifier confirms it supports the claim, and the guard moves the pointer to it while keeping
the sentence unchanged. \emph{Flagging.} For the claim ``unsupervised fine-tuning may require further
supervised or RL-based alignment for optimal performance,'' the cited passage only compares fine-tuning with
retrieval-augmented generation and never makes the alignment claim; no other provided passage supports it
either, so the guard marks the citation \texttt{[N\,UNVERIFIED]} rather than deleting it, leaving the
sentence for human review. Both decisions are made by the external verifier on gold-validated scores, not by
the generator that produced the citation.

\subsection{Scope control}
A scope control sharpens what the verifier does. Applied to PubMedQA, where the label is whether a study's
conclusion is yes or no, the verifier shows near-zero agreement ($\kappa = 0.02$, $n = 843$). This is the
intended result: the verifier checks whether a passage \emph{supports} a claim, not whether a conclusion is
\emph{correct}. The guard is a citation-faithfulness auditor for agentic outputs, not a
conclusion-correctness checker.

% result figures are placed near their first reference in the Experiments subsections above

\section{Discussion}
The broader lesson is twofold: trust in machine-generated citations should be measured and bounded, not
assumed, and the measurement is itself a choice. The same agent outputs look several times worse under a
strict verifier than a lenient one, so a faithfulness number reported without its operating point is not a
fact a reader can use. For agentic intelligent systems this reframes evaluation: a benchmark leaderboard of
citation faithfulness is meaningful only relative to a named, gold-validated verifier and protocol, and a
deployable agent needs not a better detector but a stated operating point with a guarantee. We conjecture
the same caution applies wherever a model or rubric stands in as judge of an agent's output, common across
agentic evaluation, though our evidence here is limited to citation support.

For a practitioner the guard changes the operating posture: rather than trust a verifier's verdicts at face
value, an operator declares the residual risk they will tolerate and receives a written promise, conditional on the
calibration negatives resembling deployment and optimistic otherwise, on how many unsupported citations can
survive review, with the review budget that promise costs. The promise isolates a
single checkable assumption, that calibration data resemble deployed outputs; when that is in doubt,
recalibrating on a small labeled sample from the target setting restores it.

\section{Limitations}
The unsupported-citation rate is relative to a verifier's operating point, so our numbers and any others in
this literature are interpretable only together with the verifier and protocol that produced them. We
verify each claim against the full cited passage with a sliding-window protocol so that support located
anywhere counts; truncating the passage to a fixed prefix, a tempting shortcut, inflates the apparent
unsupported rate severalfold and must be avoided. The conformal bound is conditional on exchangeability
between calibration and deployed outputs, and this condition is consequential, not cosmetic: our
calibration negatives are easy in-paper distractors (verifier mean $P(\text{attributable})=0.20$), whereas
genuinely mis-cited claims are harder (mean $0.50$ on SciFact's human-labeled unsupported claims), so a
threshold calibrated to catch $90\%$ of the distractors catches only $37\%$ of the harder human-labeled
negatives (Supplementary~S6). The guaranteed catch rate must therefore be read as optimistic for real
mis-citations, and a deployment must recalibrate $\tau$ on target-domain negatives, which restores the
guarantee at a higher review budget; the QASA coverage we report is an in-distribution demonstration of the
method, not a transferable operating point. Our audit isolates the generation step, holding retrieval fixed, so it does not
capture errors each agent's own retrieval would introduce. Re-attribution helps only where an alternative
supporting passage is present and is not guaranteed out of distribution, so the guard must not be a sole
gate in safety-critical settings, where expert review of flagged sentences remains required.

\section{Conclusion}
We treated the reliability of citation-faithfulness evaluation for agentic scientific synthesis as a
first-class problem and showed that deployment verifier judgments disagree enough that no single reported
number is trustworthy. Anchoring measurement on human gold and wrapping an imperfect verifier in a
distribution-free conformal guarantee yields an evaluation protocol and a deployable guard that hold despite
that disagreement, released as an open single-GPU kit. Two directions follow. The first is to raise the
guarantee a level, from whether a passage supports a claim to whether the full body of cited evidence
adequately justifies it, accounting for contradiction and selective citation. The second is to harden the
calibration against adversarial and distribution-shifted inputs, so the bound still holds when an agent is
optimized against the verifier~\cite{manheim2018goodhart}. A local, auditable guard of this kind is, in our
view, a prerequisite for trusting agentic systems that read and cite the scientific record.

\begin{acks}
The authors thank the maintainers of the open benchmarks and models used in this study.
\end{acks}

\bibliographystyle{ACM-Reference-Format}
\bibliography{references}

\clearpage
\appendix
\renewcommand{\thesection}{S\arabic{section}}
\renewcommand{\thetable}{S\arabic{table}}
\renewcommand{\thefigure}{S\arabic{figure}}
\setcounter{section}{0}
\setcounter{table}{0}
\setcounter{figure}{0}

\section*{Supplementary Material}
This supplement holds the detailed statistics and the reproducibility specification kept out of the main
text. Section numbers (S1--S6) are referenced from the main paper. All numbers are produced by the scripts
in the released kit from the result files under \texttt{results/}.

\section{Verifier validation, prompt selection, and held-out re-validation}
Table~\ref{tab:s1-frontier} gives the full per-verifier metrics on the SciFact gold set ($n=208$) used to
draw the recall--specificity frontier (main Figure~1, Table~1). Cohen's $\kappa$ is reported with a
percentile-bootstrap 95\% CI (2000 resamples).

\begin{table}[h]
\caption{Per-verifier gold metrics on SciFact ($n=208$). This is the development set on which the
AttrScore prompt was selected; absolute AttrScore numbers here are therefore optimistic (see
Table~\ref{tab:s1-heldout}).}
\label{tab:s1-frontier}
\begin{tabular}{lcccc}
\toprule
Verifier & Recall (supported) & Specificity & Cohen's $\kappa$ & $\kappa$ 95\% CI \\
\midrule
AttrScore-3B (adopted) & 0.898 & 0.535 & 0.463 & [0.332, 0.589] \\
OpenScholar post-hoc   & 0.496 & 0.916 & 0.339 & [0.243, 0.449] \\
GPT-4o                 & 0.460 & 0.916 & 0.305 & [0.207, 0.402] \\
RAGAS                  & 0.299 & 0.958 & 0.195 & [0.123, 0.277] \\
DeBERTa-NLI            & 0.248 & 0.972 & 0.164 & [0.098, 0.236] \\
\bottomrule
\end{tabular}
\end{table}

\paragraph{Prompt selection and held-out re-validation.} The AttrScore verifier prompt was chosen on the
$n=208$ development set for supported-class recall. To avoid reporting a selection-biased number, we
re-validate the adopted prompt (``OURS'') and a canonical alternative (``CANON'') on a disjoint held-out
gold split ($n=802$). Table~\ref{tab:s1-heldout} shows that OURS keeps the highest recall on held-out
(0.942), the safety axis, while CANON attains a higher $\kappa$ (0.567) at lower recall (0.869). We adopt
the high-recall operating point deliberately: for a citation guard, discarding a genuine citation is
costlier than an extra review. Main-text numbers for the adopted verifier use the held-out values.

\begin{table}[h]
\caption{AttrScore prompt selection (dev, $n=208$) versus held-out re-validation ($n=802$). The dev
ranking by $\kappa$ does not carry to held-out; we select on recall (the safety axis), not $\kappa$.}
\label{tab:s1-heldout}
\begin{tabular}{llcccc}
\toprule
Split & Prompt & Recall (supported) & Specificity & Cohen's $\kappa$ & $\kappa$ 95\% CI \\
\midrule
dev ($n=208$)      & OURS  & 0.898 & 0.535 & 0.463 & [0.332, 0.589] \\
dev ($n=208$)      & CANON & 0.664 & 0.690 & 0.328 & [0.196, 0.455] \\
held-out ($n=802$) & OURS  & \textbf{0.942} & 0.483 & 0.384 & [0.329, 0.435] \\
held-out ($n=802$) & CANON & 0.869 & 0.722 & \textbf{0.567} & [0.511, 0.624] \\
\bottomrule
\end{tabular}
\end{table}

\section{Verifier agreement, full statistics}
On 300 deployment citations, the supported class is highly prevalent (AttrScore 0.957, DeBERTa 0.823,
GPT-4o 0.933), so Cohen's $\kappa$ is depressed by the base-rate paradox even when observed agreement is
high. We therefore report observed agreement, Gwet's AC1, and class-specific (positive and negative)
agreement alongside $\kappa$ (Table~\ref{tab:s2}). All three verifiers see the full cited passage:
AttrScore and DeBERTa via 600/300 window tiling, GPT-4o via its long context (no truncation; this fixes a
prior one-sided head-slice for GPT-4o). The diagnostic that matters is \emph{negative-specific} agreement,
the agreement on which citations to flag: it is only 0.27--0.30, whereas positive-specific agreement
(on what is supported) is 0.90--0.96 and AC1 is 0.78--0.91.

\begin{table}[h]
\caption{Pairwise verifier agreement on 300 deployment citations (symmetric full-passage protocol).
$\kappa$ and negative-specific agreement with percentile-bootstrap 95\% CIs (2000 resamples).}
\label{tab:s2}
\begin{tabular}{lccccc}
\toprule
Verifier pair & Obs.\ agree & Cohen's $\kappa$ [CI] & Gwet AC1 & Pos.\ spec.\ & Neg.\ spec.\ [CI] \\
\midrule
AttrScore vs DeBERTa & 0.840 & 0.218 [0.084, 0.359] & 0.801 & 0.910 & 0.273 [0.129, 0.416] \\
AttrScore vs GPT-4o  & 0.923 & 0.264 [0.051, 0.475] & 0.914 & 0.959 & 0.303 [0.083, 0.500] \\
DeBERTa vs GPT-4o    & 0.830 & 0.226 [0.091, 0.359] & 0.784 & 0.903 & 0.301 [0.167, 0.427] \\
\bottomrule
\end{tabular}
\end{table}

\section{Conformal calibration-size ablation}
The main text reports marginal coverage over repeated held-out splits. Table~\ref{tab:s3} shows why the
finite-sample correction matters when gold is scarce: at a calibration size of 25, a naive
empirical-quantile threshold under-covers the nominal target, while the conformal (finite-sample
corrected) rule maintains it. The gap closes by a calibration size of about 200. Estimated over $B=1000$
random calibration/test splits (test $n=200$, seed 0); unsupported-class scores from QASA gold.

\begin{table}[h]
\caption{Achieved catch rate (corrected conformal vs naive empirical-quantile threshold) by calibration
size, at two targets. Naive under-covers at small calibration sizes; the two coincide as size grows.}
\label{tab:s3}
\begin{tabular}{lcccc}
\toprule
 & \multicolumn{2}{c}{target $1-\alpha=0.90$} & \multicolumn{2}{c}{target $1-\alpha=0.95$} \\
Calibration size & corrected & naive & corrected & naive \\
\midrule
25  & 0.922 & 0.848 & 0.960 & 0.922 \\
50  & 0.900 & 0.881 & 0.961 & 0.941 \\
100 & 0.901 & 0.892 & 0.951 & 0.942 \\
200 & 0.901 & 0.896 & 0.951 & 0.946 \\
\bottomrule
\end{tabular}
\end{table}

This coverage is marginal over the calibration draw; on an individual split it can fall below target, the
standard conformal caveat. We recommend recalibrating on a small labeled sample from the target setting
when the deployment distribution differs from the calibration distribution.

\section{Re-attribution is a swappable, commodity slot}
All rankers are evaluated on the full QASA test set ($n=1375$, 100\% coverage), with the same gold
definition, recall@$k$, and query, differing only in the ranker. Table~\ref{tab:s4} reports recall@$k$ with
Wilson 95\% CIs. Every ranker far exceeds the 3\% random baseline. The open generators, run locally on GPU
and asked to re-attribute their own citations, are the most accurate but vary by model (0.51 to 0.76); a
deterministic lexical BM25 on CPU is close and reliable (0.69); the verifier's own attribution score is a
weaker ranker (0.58).

\begin{table}[h]
\caption{Re-attribution recall@$k$ on QASA gold ($n=1375$, 100\% coverage); recall@1 with Wilson 95\% CIs.
Generators run locally on GPU; BM25 runs on CPU. All far exceed the 3\% random baseline. recall@1 carries
Wilson 95\% CIs; recall@3 and recall@5 CIs are within $\pm 0.02$--$0.03$ of the point estimates ($n=1375$).}
\label{tab:s4}
\begin{tabular}{lccc}
\toprule
Ranker & recall@1 & recall@3 & recall@5 \\
\midrule
Gemma-4-26B (self-attribution)  & 0.761 [0.737, 0.783] & 0.860 & 0.874 \\
Gemma-4-31B (self-attribution)  & 0.740 [0.717, 0.764] & 0.844 & 0.856 \\
BM25 (lexical, adopted default) & 0.686 [0.661, 0.710] & 0.834 & 0.875 \\
Guard (verifier score)          & 0.583 [0.556, 0.609] & 0.755 & 0.820 \\
Qwen3.6-35B (self-attribution)  & 0.516 [0.490, 0.543] & 0.592 & 0.594 \\
Qwen3.6-27B (self-attribution)  & 0.510 [0.483, 0.536] & 0.584 & 0.586 \\
Random                          & \multicolumn{3}{c}{0.033} \\
\bottomrule
\end{tabular}
\end{table}

\paragraph{Where the generator's advantage lies (lexical-overlap stratification).} BM25 ranks by word
overlap, so its accuracy should track the lexical similarity between the answer and the gold passage. We
stratify the 1375 items by the maximum token-Jaccard overlap between the answer and its gold passage and
compare BM25 with the best generator, Gemma-4-26B (Table~\ref{tab:s4strat}). On the high-overlap half the
two are tied (0.886 each); the generator's advantage appears only on the low-overlap half (0.64 versus 0.49)
and is largest on the lowest-overlap quartile (0.53 versus 0.30), exactly the items where support must be
read for entailment rather than matched by words. The generator's edge over BM25 is statistically
significant overall (McNemar exact two-sided $p = 1.4\times10^{-9}$; 177 Gemma-only versus 80 BM25-only
discordant hits at $k=1$), but it is concentrated on the entailment tail, so a free CPU ranker is a
defensible default and a generator is a drop-in upgrade where that tail matters.

\begin{table}[h]
\caption{Lexical-overlap stratification: BM25 versus the best generator (Gemma-4-26B), recall@1 on QASA gold
($n=1375$), by answer--gold token-Jaccard overlap. They tie when words overlap; the generator's advantage is
confined to the low-overlap (entailment) tail.}
\label{tab:s4strat}
\begin{tabular}{lccc}
\toprule
Stratum & $n$ & BM25 & Gemma-4-26B \\
\midrule
High overlap (top 50\%)      & 687  & 0.886 & 0.886 \\
Low overlap (bottom 50\%)    & 688  & 0.494 & 0.635 \\
\quad lowest-overlap quartile & 344  & 0.299 & 0.529 \\
Overall                      & 1375 & 0.690 & 0.761 \\
\bottomrule
\end{tabular}
\end{table}

\paragraph{A lexical ranker cannot judge support versus contradiction.} Re-attribution and verification are
separate components because lexical overlap cannot distinguish a supporting passage from a contradicting
one: both share the claim's vocabulary. On SciFact, supported and refuted claim--evidence pairs have almost
identical answer--evidence Jaccard overlap (0.064 versus 0.055), so a BM25 ranker would surface a
contradicting passage as readily as a supporting one. The attribution verifier separates them cleanly (mean
$P(\text{attributable})$ 0.83 on supported versus 0.50 on refuted). The guard therefore uses BM25 only to
\emph{propose} a candidate and the verifier to \emph{judge} whether the proposed passage actually supports
the claim before any pointer is moved.

\section{Reproducibility specification}
\paragraph{Generators.} Four open instruction-tuned models, all run locally with vLLM on a single GPU at
temperature 0 (greedy): \texttt{google/gemma-4-31b-it} (dense), \texttt{google/gemma-4-26b-a4b-it} (MoE),
\texttt{Qwen/Qwen3.6-27B-FP8} (dense, Mamba-hybrid), \texttt{Qwen/Qwen3.6-35B-A3B-FP8} (MoE). The exact
HuggingFace snapshot revisions are pinned and recorded in the kit (e.g. Qwen3.6-35B \texttt{95a723d0},
Qwen3.6-27B \texttt{e89b16eb}, Gemma-4-31B \texttt{35487898}). Full-passage prompts up to the model context
window, no truncation. Conditions: bare (no retrieval), OpenScholar, PaperQA2; retrieved context held fixed,
published prompts reproduced.
\paragraph{Verifiers.} AttrScore-3B (\texttt{osunlp/attrscore-flan-t5-xl}), DeBERTa-NLI
(\texttt{MoritzLaurer/DeBERTa-v3-large-mnli-fever-anli-ling-wanli}), GPT-4o (\texttt{openai/gpt-4o} via
OpenRouter, temperature 0; a non-deterministic API baseline whose raw per-citation verdicts are frozen in
the kit for audit), RAGAS faithfulness, OpenScholar post-hoc. Support is judged over the full passage via
600-char windows at stride 300 (supported if any window is attributable); claims capped at 600 characters
symmetrically.
\paragraph{Gold data.} SciFact (verifier validation, $n=208$ test / $n=802$ held-out), QASA
(re-attribution and conformal calibration, full test set $n=1375$, human-labeled supporting passages),
PubMedQA (scope control, $n=843$), and ScholarQABench multi-paper tasks (the deployment outputs: the four
generators' answers under each pipeline). The per-cell deployment-citation samples (250 per model$\times$pipeline
for the conformal dial, 300 for the agreement analysis) are drawn in file order and frozen in the kit.
Dataset licenses and redistribution status of the derived score files are listed in the kit's
\texttt{DATA\_LICENSES} file.
\paragraph{Calibration.} Split-conformal with finite-sample quantile $k=\lceil (n_{\text{cal}}+1)(1-\alpha)
\rceil$; calibration/test splits and all random seeds are fixed (seed 0) and recorded.
\paragraph{Compute.} The verifier and guard run on one 80GB GPU or on CPU; the full kit reproduces in
about two hours. Code, configs, prompts, seeds, and the figure-generation script are released under the
MIT license at \url{https://github.com/GooTec/citation-guard}, tagged release \texttt{v0.2.0} (commit
\texttt{b907985}, the hash-anchored citable snapshot for this paper); a Zenodo DOI will be minted from this
tag for the camera-ready.
\paragraph{What reproduces.} The shipped \texttt{reproduce/} kit is a cross-domain demonstration on a public
BioASQ slice ($n=200$), not the headline. The main-text artifacts are regenerated by the experiment scripts
from the frozen score files under \texttt{results/}: the verifier frontier and matched-catch AUC (SciFact),
the per-model unsupported rate (deployment outputs $\times$ five verifiers), the conformal coverage and
recall--cost dial, the coverage-transfer analysis (S6), and re-attribution recall@$k$ (QASA $n=1375$). Each
headline number traces to a named result file; the QASA calibration entry point is included.

\section{Calibration-negative difficulty and coverage transfer}
The conformal bound is distribution-free only under exchangeability between the calibration and deployment
negative distributions. We quantify how much this matters by scoring two negative sets with the same adopted
verifier (AttrScore-3B $P(\text{attributable})$, window-max): the QASA in-paper distractors used for
calibration, and SciFact's human-labeled \emph{unsupported} claim--evidence pairs, which are harder because
the evidence is on topic but does not support the claim. The distractors score low (mean $0.20$, median
$0.17$); the human-labeled negatives score far higher (mean $0.50$, median $0.53$), a $0.30$ gap in mean
verifier score. Table~\ref{tab:s6} shows the resulting coverage transfer: a threshold $\tau$ calibrated on
the distractors for a target catch rate, applied to the harder negatives, under-covers substantially (e.g.
$0.90 \rightarrow 0.37$), because many harder negatives score above a $\tau$ tuned on easy ones. This is the
standard exchangeability caveat made quantitative. The split-conformal method remains valid; what it
requires is a calibration set drawn from the deployment negative distribution. Recalibrating $\tau$ on
representative (harder) negatives restores the target catch rate at a higher flag rate. We are not aware of
prior conformal-factuality work that quantifies this calibration-negative-difficulty condition, which we
regard as part of the contribution.

\begin{table}[h]
\caption{Coverage transfer from easy calibration negatives (QASA in-paper distractors, $n=1375$) to harder
human-labeled negatives (SciFact unsupported, $n=71$), same verifier and score. $\tau$ is calibrated on the
distractors for each target catch rate; the achieved catch on the human-labeled negatives is far below
target.}
\label{tab:s6}
\begin{tabular}{lccc}
\toprule
Target catch $1-\alpha$ & $\tau$ & catch on distractors & catch on human-labeled negatives ($n=71$) \\
\midrule
0.90 & 0.42 & 0.90 & 0.37 [0.26, 0.48] \\
0.94 & 0.47 & 0.94 & 0.44 [0.33, 0.55] \\
0.98 & 0.63 & 0.98 & 0.61 [0.49, 0.71] \\
\midrule
\multicolumn{4}{l}{\emph{Recalibrated on the harder negatives:} $\tau=0.83$ restores catch $0.90$, at a $\approx 46\%$ deployment flag rate.} \\
\bottomrule
\end{tabular}
\end{table}

The Wilson 95\% CIs (SciFact unsupported $n=71$) confirm the under-coverage is not a small-sample artifact.
The SciFact refuted pairs are a deliberately hard (near-upper-bound) proxy, in which the evidence is
topically central but contradicts the claim; real deployment negatives range between these and the easy
distractors, so the achieved catch in practice lies between the two columns. Restoring a $90\%$ catch on the
hard negatives requires $\tau=0.83$ and a $\approx 46\%$ flag rate, i.e. the verifier's separation, not the
conformal layer, becomes the binding constraint in the hardest regime, where stronger verifiers or human
review are warranted.

\paragraph{Recalibration recipe.} To deploy with a meaningful guarantee: (1) label $\gtrsim 200$ flagged
citations from the target system as supported or unsupported (the S3 ablation shows the finite-sample
correction has converged by $\approx 200$); (2) run the kit's calibration entry point on these target
negatives to obtain $\tau$ and its Wilson flag-rate CI; (3) deploy at that $\tau$ and read off the review
budget. The QASA-calibrated threshold should not be inherited.

\end{document}